\documentclass[conference]{IEEEtran}
\IEEEoverridecommandlockouts
\usepackage{cite}
\usepackage{url}
\usepackage{amsmath,amssymb,amsfonts}
\usepackage{algorithmic}
\usepackage{graphicx}
\usepackage{textcomp}
\usepackage{xcolor}
\usepackage{booktabs}
\usepackage{hyperref}
\usepackage{cleveref}
\usepackage{amsthm}
\usepackage{mathtools}

\newtheorem{proposition}{Proposition}
\newtheorem*{assumption}{Assumption}

\usepackage{booktabs}
\usepackage{makecell}
\newcommand{\NA}{\multicolumn{1}{c}{\textemdash}}
\newcommand{\triple}[3]{\makecell[c]{#1\\#2\\#3}}

\usepackage{stfloats}
\usepackage{balance}

\def\BibTeX{{\rm B\kern-.05em{\sc i\kern-.025em b}\kern-.08em
    T\kern-.1667em\lower.7ex\hbox{E}\kern-.125emX}}
\begin{document}

\title{When Learning Rates Go Wrong:\\
Early Structural Signals in PPO Actor–Critic}

\author{
\makebox[.33\linewidth][c]{
\begin{tabular}{c}
Alberto Fernández-Hernández*\thanks{*Corresponding author} \\
\textit{Universitat Politècnica de València} \\
Valencia, Spain \\
a.fernandez@upv.es
\end{tabular}
}
\makebox[.33\linewidth][c]{
\begin{tabular}{c}
Cristian Pérez-Corral\\
\textit{Universitat Politècnica de València} \\
Valencia, Spain \\
cpercor@upv.es
\end{tabular}
}
\makebox[.33\linewidth][c]{
\begin{tabular}{c}
Jose I. Mestre \\
\textit{Universitat Politècnica de València} \\
Valencia, Spain \\
jmiravet@uji.es
\end{tabular}
}\\[1em]
\makebox[.33\linewidth][c]{
\begin{tabular}{c}
Manuel F. Dolz \\
\textit{Universitat Jaume I} \\
Castelló de la Plana, Spain \\
dolzm@uji.es
\end{tabular}
}
\makebox[.33\linewidth][c]{
\begin{tabular}{c}
Jose Duato \\
\textit{Openchip \& Software Technologies} \\
Barcelona, Spain \\
jose.duato@openchip.com
\end{tabular}
}
\makebox[.33\linewidth][c]{
\begin{tabular}{c}
Enrique S. Quintana-Ortí \\
\textit{Universitat Politècnica de València} \\
Valencia, Spain \\
quintana@disca.upv.es
\end{tabular}
}
}

%\author{\IEEEauthorblockN{Anonymous Authors}
%\IEEEauthorblockA{Affiliation withheld for review}}

\maketitle
\begin{abstract}
Deep reinforcement learning systems are highly sensitive to the learning rate (LR), and selecting stable and performant training runs often requires extensive hyperparameter search. 
In Proximal Policy Optimization (PPO) actor--critic methods, small LR values lead to slow convergence, whereas large LR values may induce instability or collapse. 
We analyse this phenomenon from the behavior of the hidden neurons in the network using the Overfitting-Underfitting Indicator (OUI), a metric that quantifies the balance of binary activation patterns over a fixed probe batch. 
We introduce an efficient batch-based formulation of OUI and derive a theoretical connection between LR and activation sign changes, clarifying how a correct evolution of the neuron's inner structure depends on the step size.

Empirically, across three discrete-control environments and multiple seeds, we show that OUI measured at only 10\% of training already discriminates between LR regimes. 
We observe a consistent asymmetry: critic networks achieving highest return operate in an intermediate OUI band (avoiding saturation), whereas actor networks achieving highest return exhibit comparatively high OUI values.
We then compare OUI-based screening rules against early return, clip-based, divergence-based, and flip-based criteria under matched recall over successful runs.
In this setting, OUI provides the strongest early screening signal: OUI alone achieves the best precision at broader recall, while combining early return with OUI yields the highest precision in best-performing screening regimes, enabling aggressive pruning of unpromising runs without requiring full training.
\end{abstract}

\begin{IEEEkeywords}
reinforcement learning, PPO, actor--critic, OUI, learning rate, internal metrics, stability
\end{IEEEkeywords}
\section{Introduction}

Reinforcement Learning (RL) \cite{sutton2018reinforcement} studies how an agent can learn to make sequential decisions through interaction with an environment, improving its behavior from scalar reward signals rather than supervised labels. 
The objective is to maximize the expected cumulative reward obtained when following a given parametric policy. 
In modern deep RL, both policy and auxiliary estimators are represented by neural networks, whose training dynamics are highly sensitive to hyperparameters, particularly the learning rate (LR).

Actor--critic methods \cite{konda2000actor} constitute one of the most successful frameworks for continuous control and high-dimensional problems. 
In this setting, the \emph{actor} is a neural network receiving a state as input and outputting a distribution over actions. 
The \emph{critic} is a second network, typically mapping the state to a scalar estimate of the expected return. 
The critic provides a learning signal to the actor, while the actor’s evolving policy determines the distribution of states on which the critic is trained. 
Both networks are, therefore, tightly coupled: modifying the actor changes the critic’s data distribution, and modifying the critic alters the update direction of the actor.

Proximal Policy Optimization (PPO) \cite{schulman2017ppo} stabilizes actor updates by constraining how far the policy can move at each iteration. Despite this mechanism, the LR remains a critical factor. 
If too small, learning progresses slowly and may stall; if too large, updates can induce abrupt representation changes, degrade value estimation, and ultimately collapse performance. 
Understanding how LR affects not only return but also the internal behavior of the networks is therefore of practical and theoretical interest.

In this work, we study the effect of LR through an internal metric, the Overfitting-Underfitting Indicator (OUI) \cite{OUI_reference}. Rather than evaluating performance solely through external signals such as return or loss values, OUI quantifies how effectively the network utilizes its internal representational capacity on a fixed probe set of states. In this sense, OUI diagnoses at each stage of training the \emph{structure} of the network, assigning a value related to the richness and variability of the neurons' behaviors across inputs. 

Intuitively, it measures how evenly different neurons participate in partitioning a batch of inputs into activation regions: high OUI indicates balanced and distributed usage of neuron activations, whereas low OUI reflects structural saturation, with many neurons behaving almost uniformly.

This perspective allows us to analyse training runs from inside the model, that is, by looking at the behavior of the hidden neurons of the network.
Instead of asking only whether a given LR achieves high return, we ask how it shapes the structural evolution of actor and critic representations, and whether early structural signals can help discriminate between stable and unstable regimes.

\medskip

Our central question is therefore: \emph{how does the LR affect the internal organisation of actor--critic networks, and can early structural measurements guide the selection of stable and high-performing training runs?}

\medskip
The main contributions of this work are the following:

\begin{itemize}
\item We introduce a batch-based formulation of the Overfitting--Underfitting Indicator (OUI) suitable for probing the internal structure of actor--critic networks during RL training.

\item We derive a theoretical connection between the learning rate, activation sign changes, and the evolution of OUI, providing a structural interpretation of how gradient step size affects internal network organization.

\item We empirically show across three discrete-control environments that OUI measured at only 10\% of training already discriminates between LR regimes, revealing a consistent asymmetry between actor and critic structural behavior.

\item We demonstrate that OUI provides a strong early screening signal for RL training runs, outperforming commonly monitored PPO signals such as early return, Kullback-Leibler (KL divergence, clipping statistics, and activation flip rates under matched recall.
\end{itemize}

\medskip
The organization of the article is as follows: Section~\ref{sec:related_work} reviews related work on RL stability and internal network diagnostics. 
Section~\ref{sec:theory} formalizes the relationship between learning rate, activation flips, and OUI dynamics. 
Section~\ref{sec:experiments} presents the experimental protocol and empirical results across environments. 
Finally, Section~\ref{sec:conclusion} summarizes the findings and discusses limitations and future directions.

\section{Related Work}
\label{sec:related_work}

RL performance is well known to be sensitive to optimization and implementation details, including LR and other hyperparameters \cite{henderson2018deep}, often yielding high variance across random seeds and unstable conclusions if protocols are not carefully controlled. 
In policy-gradient methods, PPO \cite{schulman2017ppo} partially mitigates instability by constraining policy updates, yet practitioners routinely monitor additional signals (e.g., KL divergence \cite{KL}, clipping statistics) as training can still enter failure regimes depending on the LR and data non-stationarity.

A separate line of work studies internal network behavior through activation- and representation-level measurements, aiming to characterize how learned neurons evolve beyond external task metrics. 
Classical tools include representational similarity measures (e.g., CKA \cite{kornblith2019similarity} and SVCCA \cite{raghu2017svcca}) and activation statistics that quantify how neurons align, collapse, or remain diverse during training. 
While such analyses are widespread in supervised learning, their systematic use as practical probes for deep RL stability remains limited, despite RL exhibiting pronounced sensitivity to training dynamics.

OUI \cite{OUI_reference} was introduced as a lightweight activation-based metric to quantify how evenly a network uses its neuron activations on a fixed probe set. 
Prior work considered a pairwise formulation and focused on supervised classification, studying its relationship with regularisation mechanisms such as weight decay. 
In this work, we adapt OUI to an lightweight batch-based formulation and apply it to PPO actor--critic systems, using OUI as an early probe to analyse how the LR shapes the behavior of the hidden layers differently in the actor and the critic.

\section{Theory: LR, Flips, and OUI Dynamics}
\label{sec:theory}

We now formalize the structural mechanism through which the LR shapes internal representations. 
Our objective is to connect the three elements: 
(i) gradient step size (LR), 
(ii) activation sign changes, and 
(iii) evolution of OUI.

\subsection{Activation Patterns and OUI Definition}

Fix a probe batch of $B$ states $S_{\text{probe}}=\{x_1,\dots,x_B\}$, which remains constant throughout the analysis. 
Consider a given network architecture and a layer $l$ with $d_l$ neurons. Let $\theta$ denote the network parameters. For each input $x_b$ in the probe batch, let 
$A^{(l)}(x_b;\theta)\in\mathbb{R}^{d_l}$ be the vector of preactivations at layer $l$ obtained from the forward pass.

We define the binary activation mask
\[
M^{(l)}_{b,j}(\theta)
=
\mathbf{1}\{A^{(l)}_j(x_b;\theta) > 0\},
\]
for $b\in\{1,\dots,B\}$ and $j\in\{1,\dots,d_l\}$, indicating whether neuron $j$ is active on input $x_b$.

For each neuron $j$, let
\[
s_j(\theta) \coloneqq \sum_{b=1}^{B} M^{(l)}_{b,j}(\theta)
\quad\text{and}\quad
p_j(\theta) \coloneqq \frac{s_j(\theta)}{B}
\]
denote respectively the number and the fraction of probe inputs that activate neuron $j$.

The batch-based OUI at layer $l$ is then defined as
\begin{equation}
\mathrm{OUI}^{(l)}(\theta)
=
\frac{1}{d_l}
\sum_{j=1}^{d_l}
\frac{\min(s_j(\theta), B - s_j(\theta))}{\lfloor B/2 \rfloor}
\in [0,1].
\label{eq:oui_def}
\end{equation}
That is, OUI averages across neurons $j \in \{1,\ldots,d_l\}$ in layer $l$ a normalized score that attains its maximum when $s_j(\theta)$ is closest to $\lfloor B/2 \rfloor$. This can be written equivalently\footnote{The equality holds exactly when $B$ is even. When $B$ is odd, the expression differs only by a multiplicative factor of $(B-1)/B$, which is negligible for our purposes.} as
\begin{equation}
\mathrm{OUI}^{(l)}(\theta)
\approx
\frac{1}{d_l}
\sum_{j=1}^{d_l}
\left(
1 - 2\big|p_j(\theta)-\tfrac12\big|
\right).
\label{eq:oui_balance}
\end{equation}
Equation ~\eqref{eq:oui_balance} shows that OUI measures how evenly each neuron activation partitions the probe batch. 
It is maximized when many neurons split the batch approximately $50/50$, and decreases when neurons become structurally biased (almost always active or inactive).

\subsection{Learning Rate and Activation Sign Flips}

We now analyse how a single gradient step affects these activation patterns. 
Consider an update $\theta^+ = \theta - \eta g$, where $\eta>0$ is the LR and $g$ is the update direction.

To understand how $\mathrm{OUI}^{(l)}$ evolves, it suffices to examine how each $p_j(\theta)$ changes. 
Fix a neuron index $j\in\{1,\dots,d_l\}$ and rewrite for brevity the scalar preactivation
\begin{equation*}
X_b(\theta) \coloneqq A^{(l)}_j(x_b;\theta)
\end{equation*}
of the sample $x_b$, and its corresponding binary gate
\begin{equation*}
G_b(\theta) \coloneqq \mathbf{1}\{X_b(\theta)>0\}.
\end{equation*}

A \emph{flip} occurs whenever $G_b(\theta^+)\neq G_b(\theta)$, that is, when the update causes $X_b$ to cross zero.

Define the signed directional derivative of $X_b$ in parameter space along the update direction $-g$:
\begin{equation*}
U_b(\theta,g) \coloneqq -\langle \nabla_\theta X_b(\theta),\, g\rangle.
\label{eq:Ub_def}
\end{equation*}

\begin{assumption}
\label{ass:regularity} We use the following mild regularity conditions at the current iterate $\theta$.
\begin{enumerate}
  \renewcommand{\labelenumi}{(A\arabic{enumi})}
  \item For each $b$, the map $\theta\mapsto X_b(\theta)$ is twice continuously differentiable in a neighbourhood of $\theta$.
  \item For each $b$, conditioned on the current training history (including the randomness used to build $g$), the random variable $X_b(\theta)$ admits a density $f_b$ that is continuous at $0$.
  \item The moment $\mathbb{E}(|U_b(\theta,g)|)$ is finite.
\end{enumerate}
\end{assumption}
These assumptions are standard in first-order small-step analyses: (A1) controls the Taylor remainder, (A2) rules out pathological ``atoms'' at the switching boundary $X_b=0$; and (A3) ensures integrability.

\begin{proposition}
\label{prop:flip_rate}
Under Assumptions (A1), (A2) and (A3), for each sample $b$,
\begin{equation*}
\mathbb{P}\!\left(G_b(\theta^+)\neq G_b(\theta)\right)
=
\eta\, f_b(0)\, \mathbb{E}\!\left(|U_b(\theta,g)|\right)+ o(\eta),
\end{equation*}
where the probability and conditional expectation are taken over the randomness in $g$ (and any other training randomness), conditioned on the current iterate $\theta$.
\end{proposition}

\begin{proof}
Fix $b$ and write $X_b\coloneqq  X_b(\theta)$ and $U_b\coloneqq U_b(\theta,g)$ for brevity. 
By (A1), a first-order Taylor expansion of $X_b(\theta^+)$ around $\theta$ gives
\begin{equation}
X_b(\theta^+) = X_b + \eta U_b + R_b,
\qquad
|R_b| \le C_b \eta^2,
\label{eq:taylor}
\end{equation}
for some constant $C_b$ depending on local second derivatives.
A flip occurs if $X_b$ and $X_b(\theta^+)$ have opposite signs, i.e.,
\begin{equation*}
X_b\big(X_b+\eta U_b+R_b\big) < 0.
\end{equation*}
Condition on the value of $U_b=u$. Ignoring the $R_b$ remainder momentarily, the event $X_b(X_b+\eta u)<0$ is equivalent to $X_b\in[-\eta u,0]$ when $u>0$ and $X_b\in[0,-\eta u]$ when $u<0$. 
Hence, using (A2) and the continuity of $f_b$ at $0$,
\begin{equation*}
\mathbb{P}(\text{flip}\mid U_b=u)
=
\int_0^{\eta|u|} f_b(\pm t)\,dt
=
\eta |u| f_b(0) + o(\eta),
\end{equation*}
where $\pm t$ indicates the corresponding side depending on the sign of $u$. 
The remainder $R_b$ in \eqref{eq:taylor} only enlarges/shrinks the interval by $o(\eta)$, thus contributing at most $o(\eta)$ to the probability by continuity of $f_b$ at $0$. 
Taking expectation over $U_b$ and using (A3) completes the proof.
\end{proof}

\medskip

Proposition ~\ref{prop:flip_rate} formalizes the geometric intuition:
the expected flip rate is proportional to the step size $\eta$, to the mass of probe samples near the activation boundary ($f_b(0)$), and to the normal component of the update (through $U_b$). 
Since
\[
U_b = -\|\nabla_\theta X_b\|\,\|g\|\cos\phi_b,
\]
the angle $\phi_b$ between $\nabla_\theta X_b$ and $g$ determines how strongly the update pushes towards the boundary.

\subsection{From Flip Dynamics to OUI Evolution}

Proposition~\ref{prop:flip_rate} concerns individual samples $b$ and a single neuron $j$, and is the formal result of this section.
The remainder of this section uses it as a starting point for a first-order heuristic analysis of how OUI evolves at the layer level; this analysis is meant to build intuition and motivate the empirical observations in Section~\ref{sec:experiments}, rather than to constitute a formal proof.

\medskip

To relate Proposition~\ref{prop:flip_rate} to OUI, we track how the positivity proportion $p_j(\theta)=\frac{1}{B}\sum_{b=1}^B G_b(\theta)$ changes after one step.
Let $N_{-+}$ be the number of indices $b$ for which $G_b$ flips from $0$ to $1$, and $N_{+-}$ the number of indices for which $G_b$ flips from $1$ to $0$. 
Then the positivity count updates as $s_j(\theta^+)=s_j(\theta)+N_{-+}-N_{+-}$, hence
\begin{equation}
p_j(\theta^+) - p_j(\theta) = \frac{1}{B}\big(N_{-+}-N_{+-}\big).
\label{eq:pj_change}
\end{equation}
Crucially, the net change in $p_j$ depends on the imbalance between the two flip directions, not on their total count. 
Therefore, it is possible to have many flips but a negligible drift in $p_j$ (or even a drift away from $1/2$). Note that $N_{-+}$ and $N_{+-}$ are correlated random variables whose joint distribution depends on the full gradient direction and the geometry of the loss landscape; we do not model this dependence here, and treat \eqref{eq:pj_change} as an accounting identity rather than a probabilistic bound.

We can now translate changes in $p_j$ into changes in $\mathrm{OUI}^{(l)}$ using the balance form \eqref{eq:oui_balance}. 
For each neuron $j$, define the per-neuron contribution
\begin{equation*}
\mathrm{OUI}_j(\theta) \coloneqq 1 - 2\big|p_j(\theta)-\tfrac12\big|.
\end{equation*}
Whenever $p_j(\theta)\neq 1/2$, this map is differentiable in $p_j$ with derivative
\begin{equation*}
\partial\, \mathrm{OUI}_j/\partial p_j(\theta)
=
-2\,\mathrm{sgn}\!\left(p_j(\theta)-\tfrac12\right).
\end{equation*}
Under the additional assumption that $\eta$ is small enough that no $p_j$ crosses $1/2$ in a single update—a condition that holds in the small-$\eta$ regime but may be violated for the largest LRs in our sweep—a first-order expansion gives the approximation
\begin{align}
\Delta \mathrm{OUI}^{(l)}
:&=
\mathrm{OUI}^{(l)}(\theta^+) - \mathrm{OUI}^{(l)}(\theta) \notag \\
&\approx
-\frac{2}{d_l}\sum_{j=1}^{d_l}
\mathrm{sgn}\!\left(p_j(\theta)-\tfrac12\right)\,\Delta p_j,
\label{eq:delta_oui_general}
\end{align}
where $\Delta p_j \coloneqq p_j(\theta^+) - p_j(\theta)$.
Equation~\eqref{eq:delta_oui_general} highlights a structural property that is not apparent from flip counts alone. 
While Proposition~\ref{prop:flip_rate} shows that the frequency of activation boundary crossings scales linearly with $\eta$, 
\eqref{eq:delta_oui_general} suggests that the direction of OUI change is governed by the signed drift of each $p_j$ relative to the balance point. 
In particular, the contribution of neuron $j$ to $\Delta \mathrm{OUI}^{(l)}$ is proportional to the projection of its drift $\Delta p_j$ onto the direction that reduces imbalance, namely $-\mathrm{sgn}(p_j-\tfrac12)$.

This picture suggests that OUI is insensitive to the total number of flips per se; it responds primarily to whether those flips collectively move neurons toward or away from structural balance. 
Large LRs increase the expected number of flips (by Proposition~\ref{prop:flip_rate}), but if the induced drift pushes several $p_j$ away from $\tfrac12$, the net effect may be a decrease in OUI despite high structural activity. 
Conversely, moderate updates may generate fewer flips yet increase OUI if their net drift reduces imbalance.

\medskip
\noindent\textbf{Heuristic principle (empirically motivated).}
\emph{The LR governs the magnitude of structural motion. The first-order analysis above suggests that the direction of OUI change is determined by the alignment between that motion and the equilibrium manifold $\{p_j=\tfrac12\}$, though a complete characterisation would require controlling the joint distribution of $(N_{-+}, N_{+-})$ across neurons, which we leave as an open theoretical question.}

\medskip

This framing helps explain why different networks can exhibit distinct OUI--LR profiles. 
Even when flip rates increase monotonically with $\eta$ (as Proposition~\ref{prop:flip_rate} predicts), the induced directional drifts $\Delta p_j$ depend on the optimization objective and data distribution. 
Actor and critic updates therefore need not affect structural balance in the same way, providing a plausible mechanism for the qualitatively different OUI trajectories observed empirically under identical LR sweeps.

\section{Experiments}
\label{sec:experiments}

In this section, we evaluate how the LR shapes performance and internal neuron activity in PPO actor--critic systems using a fully specified and reproducible protocol.

\subsection{Experimental Setup}

We consider three discrete-control benchmarks: CartPole-v1 \cite{openaigym} (4D observation, 2 actions), LunarLander-v3 \cite{openaigym} (8D observation, 4 actions), and MiniGrid-Empty-8x8-v0 \cite{minigrid} with \texttt{RGBImgPartialObsWrapper} and \texttt{ImgObsWrapper} (RGB image observation $(56,56,3)$, 7 actions). 
Actor and critic are implemented as fully separate networks. 
CartPole uses MLP actor and critic with hidden layers $(64,64)$ and ReLU activations; LunarLander uses $(128,128)$; MiniGrid uses a CNN with architecture $\texttt{Conv(16,3x3,s2)} \rightarrow \texttt{Conv(32,3x3,s2)} \rightarrow \texttt{FC(128)} \rightarrow \texttt{out}$. 
Image inputs are scaled to $[0,1]$; no online observation normalization is applied. Policies are discrete via a Categorical distribution and critics output scalar $V(s)$. All runs use PPO with $\gamma=0.99$, GAE $\lambda=0.95$, clip range $0.1$, 4 update epochs, minibatch size 256, value coefficient $0.5$, entropy coefficient $0$, gradient clipping at $0.5$, Adam optimizer with LR given by $\eta$ and $\epsilon=10^{-5}$ used for numerical stability, as usual. 
Training horizons are 120k timesteps (rollout 1024) for CartPole and LunarLander, and 20k timesteps (rollout 512) for MiniGrid.

We sweep a canonical grid of 13 logarithmically spaced LRs from $3.16\times10^{-5}$ to $3.16\times10^{-2}$. 
Each environment–LR pair is trained with 10 seeds (0–9), yielding 130 runs per environment. 
All aggregate curves report the median across seeds with interquartile range.

For each environment, a fixed probe batch $S_{\text{probe}}$ is generated once using a random policy (seed 0) and reused across runs (size 1024 for CartPole and LunarLander, 512 for MiniGrid). 
OUI is computed layer-wise as defined in Section~\ref{sec:theory} and averaged across hidden layers of each branch. 
It is evaluated every 1\% of total PPO updates, and the early structural metric corresponds to the checkpoint at 10\% of the updates. Lastly, return is measured as the moving average over the last 50 completed episodes.

\begin{figure*}[!p]
\centering
\includegraphics[width=\linewidth]{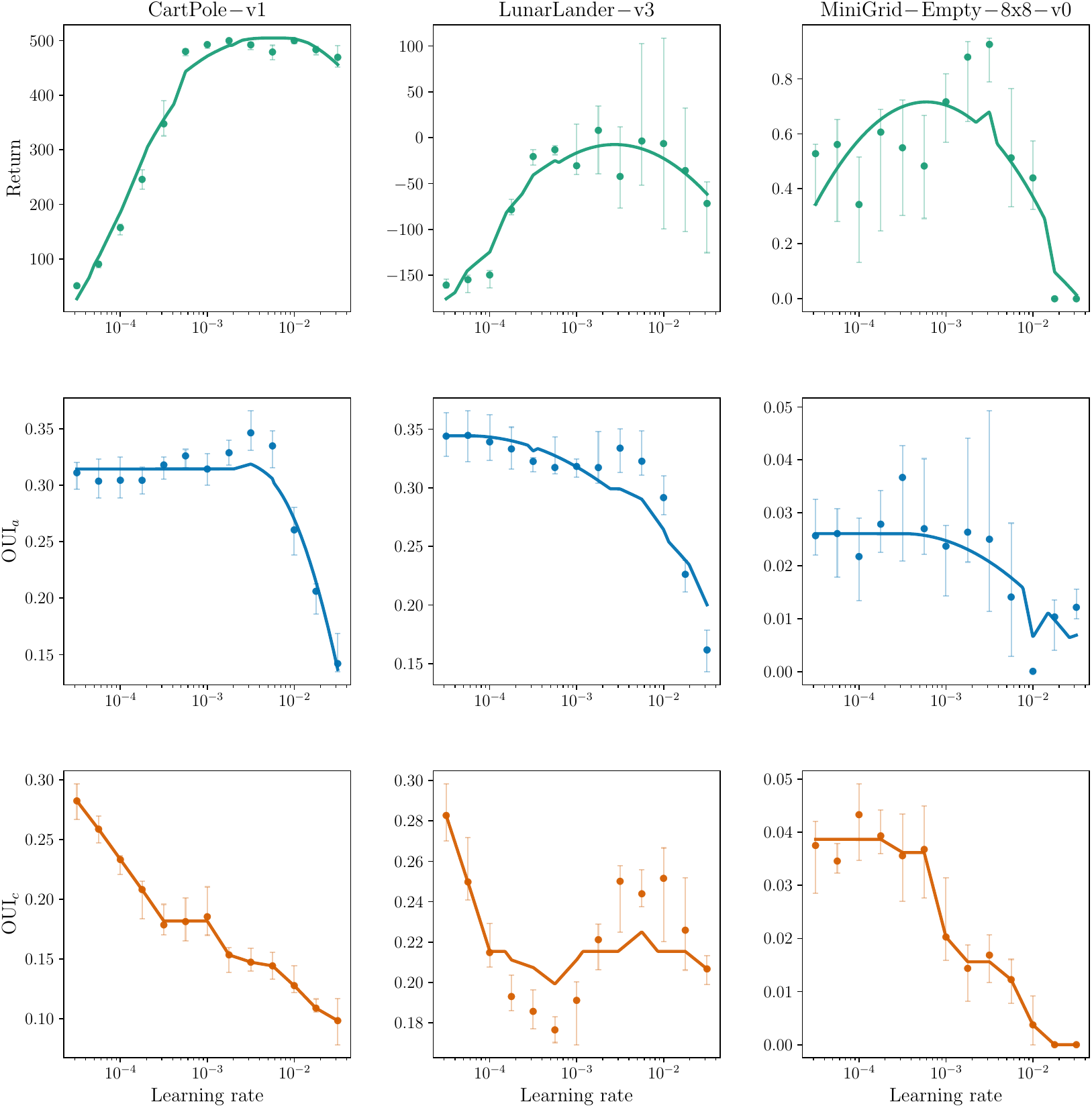}
\caption{Return (top row), actor OUI ($\textup{OUI}_a$, middle row), and critic OUI ($\textup{OUI}_c$, bottom row) as a function of LR across CartPole, LunarLander, and MiniGrid. Metrics are measured at 10\% of training. Points denote median across seeds; shaded regions indicate IQR.}
\label{fig:main}
\end{figure*}

\subsection{Structural Regimes Across Learning Rates}

The regimes observed in Figure~\ref{fig:main} admit a coherent structural explanation consistent with Section~\ref{sec:theory}. 
Increasing the LR amplifies directional drift in activation balances. 
The key distinction across regimes is whether this drift induces productive reorganisation or premature saturation.

In the under-aggressive regime (low LR), critic OUI is high but largely static: features remain distributed yet structurally inert, reflecting limited drift. 
Learning progresses, but slowly. 
In the over-aggressive regime (high LR), drift becomes excessive. 
Many critic features are rapidly pushed toward saturation ($p_j \to 0$ or $1$), reducing representational flexibility. 
Actor OUI eventually collapses as well, and return degrades sharply. 
Notably, the structural collapse of the critic typically precedes the full drop in performance.

The optimal regime lies between these extremes. 
Here, the critic exhibits measurable reorganisation without saturating, while the actor maintains consistently high OUI. 
Maximal return systematically coincides with this combination: structurally active policy representations and a critic operating in a non-saturated band. 
This aligns directly with the structural-drift picture: useful learning requires drift, but not drift that destroys diversity.

\begin{table*}[!t]
\centering
\scriptsize
\renewcommand{\arraystretch}{1.18}
\setlength{\tabcolsep}{3.5pt}
\caption{Recall-matched comparison of early screening rules. Success is defined as belonging to the top 20\% of final return within each environment. Each cell reports precision, recall, and $n_Q$, stacked vertically. Only entries with $n_Q \ge 10$ are shown. Best precision in each recall bin is highlighted in bold.}
\label{tab:recall_matched_rules}

\begin{tabular*}{\textwidth}{@{\extracolsep{\fill}}lccccccccccc@{}}
\toprule
\thead{Recall bin} & Ret. & OUI & KL & Clip & Div. & Flip & \thead{Ret.+OUI} & \thead{Ret.+KL} & \thead{Ret.+Clip} & \thead{Ret.+Div.} & \thead{Ret.+Flip} \\
\midrule
$(0.00,0.05]$
& \NA
& \NA
& \triple{0.01}{0.03}{153}
& \triple{0.03}{0.05}{116}
& \NA
& \triple{0.09}{0.04}{34}
& \triple{\textbf{0.40}}{0.05}{10}
& \triple{0.36}{0.05}{11}
& \triple{\textbf{0.40}}{0.05}{10}
& \triple{\textbf{0.40}}{0.05}{10}
& \triple{\textbf{0.40}}{0.05}{10} \\
\midrule
$(0.05,0.10]$
& \NA
& \triple{0.58}{0.09}{12}
& \triple{0.04}{0.07}{173}
& \NA
& \triple{0.62}{0.10}{13}
& \triple{0.42}{0.10}{19}
& \triple{\textbf{0.80}}{0.10}{10}
& \triple{0.47}{0.10}{17}
& \triple{0.47}{0.10}{17}
& \triple{0.47}{0.10}{17}
& \triple{0.46}{0.06}{11} \\
\midrule
$(0.10,0.15]$
& \triple{0.42}{0.14}{26}
& \triple{0.64}{0.11}{14}
& \triple{0.50}{0.11}{18}
& \triple{0.55}{0.14}{20}
& \triple{0.56}{0.11}{16}
& \triple{0.18}{0.15}{68}
& \triple{\textbf{0.82}}{0.11}{11}
& \triple{0.53}{0.12}{19}
& \triple{0.50}{0.14}{22}
& \triple{0.53}{0.12}{19}
& \triple{0.39}{0.12}{26} \\
\midrule
$(0.15,0.20]$
& \NA
& \triple{0.47}{0.19}{32}
& \triple{0.40}{0.20}{40}
& \triple{0.42}{0.20}{38}
& \triple{0.42}{0.20}{38}
& \triple{0.45}{0.16}{29}
& \triple{\textbf{0.59}}{0.16}{22}
& \triple{0.50}{0.20}{32}
& \triple{0.50}{0.19}{30}
& \triple{0.50}{0.20}{32}
& \NA \\
\midrule
$(0.20,0.25]$
& \triple{0.41}{0.22}{44}
& \triple{\textbf{0.47}}{0.22}{38}
& \NA
& \NA
& \triple{0.40}{0.25}{50}
& \triple{0.30}{0.22}{60}
& \triple{0.42}{0.22}{43}
& \triple{0.47}{0.25}{43}
& \triple{0.46}{0.22}{39}
& \triple{0.47}{0.25}{43}
& \NA \\
\midrule
$(0.25,0.30]$
& \triple{0.38}{0.28}{60}
& \triple{\textbf{0.50}}{0.28}{46}
& \triple{0.38}{0.28}{60}
& \triple{0.39}{0.28}{59}
& \triple{0.40}{0.28}{58}
& \triple{0.27}{0.26}{78}
& \triple{0.38}{0.27}{58}
& \triple{0.46}{0.26}{46}
& \triple{0.45}{0.26}{47}
& \triple{0.45}{0.26}{47}
& \NA \\
\bottomrule
\end{tabular*}

\vspace{3pt}
\begin{minipage}{\textwidth}
\footnotesize\textit{Notes.} Early return is defined from the leave-seed-out percentile of return at 10\% of training within each environment. OUI selects runs with high actor OUI and intermediate critic OUI. Divergence combines KL- and clip-based signals, and flip measures the fraction of units whose sign pattern on $S_{\mathrm{probe}}$ changes between evaluations.
\end{minipage}
\end{table*}

\subsection{Early Screening of Training Runs}

A crucial observation is that these structural regimes are already clearly separated at 10\% of training and remain stable for the rest of the training process.
This suggests that early measurements may be useful not only descriptively, but also as screening signals for deciding which runs are worth continuing.

To test this hypothesis, we compare OUI against early signals that are already available in PPO: \textit{return\_only}, \textit{approx\_kl}, \textit{clip\_fraction}, their conjunction (\textit{divergence}), sign-flip statistics (\textit{flip}), and conjunctions of early return with each of these signals.
Success is defined as belonging to the top 20\% of final return within each environment.
Because precision is strongly coupled to screening aggressiveness, we compare all rules under matched recall over successful runs.
Table~\ref{tab:recall_matched_rules} reports, for each recall bin, the precision of the selected subset together with its recall and support $n_Q$; only rules with $n_Q \ge 10$ are retained to avoid unstable small-support cases.

Two conclusions emerge.
First, OUI adds clear value beyond early return.
In the $(0.10,0.15]$ recall bin, \textit{return+OUI} achieves a precision of $0.818$ at recall $0.111$ with $n_Q=11$, whereas \textit{return\_only} reaches only $0.423$ at recall $0.136$ with $n_Q=26$.
The corresponding Benjamini--Hochberg FDR-corrected $q$-value over all candidate rules is $5.13\times 10^{-4}$ for \textit{return+OUI}, indicating that this high-precision regime is statistically robust.
In practical terms, this mode retains only $11$ of the $390$ runs, pruning $97.2\%$ of the search space, while $81.8\%$ of the retained runs are successful.

Second, OUI is not only useful in conjunction with return.
As a standalone structural signal, it becomes the strongest selector once broader recall is allowed.
In the $0.20$--$0.30$ recall range, \textit{OUI} alone achieves the best precision among all individual rules.
Moreover, in the $(0.15,0.20]$ bin, \textit{return+OUI} still reaches a precision of $0.591$ at recall $0.160$ with $n_Q=22$, outperforming the best non-OUI alternative (approximately $0.500$), with a Benjamini--Hochberg FDR-corrected $q$-value of $1.03\times 10^{-3}$ over all candidate rules.

Overall, OUI plays two complementary roles: it is the strongest broad-recall structural selector on its own, and it yields the highest-precision screening mode when combined with early return.
This substantially strengthens the case for OUI as a practical early criterion for selecting promising LR and seed configurations before full convergence.

\balance
\section{Conclusion}
\label{sec:conclusion}

We introduced a structural perspective on LR selection in PPO actor--critic systems. 
Through a batch-based formulation of OUI and an analysis of its behavior under gradient updates, we showed that the LR controls the magnitude of structural motion, while OUI captures its direction in terms of representational balance.

Across environments and seeds, experiments consistently reveal three regimes already visible at 10\% of training: structural inertia at low LRs, productive reorganisation at intermediate values, and structural collapse at high LRs. 
The highest-performing runs systematically occupy the intermediate regime, characterized by sustained actor OUI and a critic that reorganises without saturating. 
This asymmetric structural behavior between actor and critic provides a concrete signature of stable learning. 

Building on this observation, we evaluated OUI as an early screening signal under matched recall against alternatives derived from early return, KL, clipping, divergence, and activation flips.
The resulting picture is sharper than a single enrichment statistic.
OUI alone is the strongest standalone selector at broader recall, while the conjunction \textit{return+OUI} delivers the highest-precision regime, substantially improving over early return alone in the same matched-recall range.

In particular, in a high-precision screening mode, \textit{return+OUI} retains only $11$ of the $390$ runs, of which $81.8\%$ are successful, compared with $42.3\%$ for \textit{return\_only} in the same recall bin.
This makes OUI useful in two complementary ways: as a structural selector when aiming to recover a broader fraction of successful runs, and as a precision amplifier when combined with early return to isolate a very small subset of highly promising runs.

As a result, OUI is positioned as more than a descriptive structural metric.
It provides a computationally inexpensive signal for early screening of LR and seed configurations, enabling successful pruning of unpromising runs in RL.

\section*{Limitations and Future Work}

Our analysis is restricted to PPO and to discrete-control benchmarks with relatively small networks. 
Although the structural regimes identified here are consistent across the three tasks considered, it remains to validate this idea in continuous-control domains such as MuJoCo or DMControl or in other actor--critic variants. Beyond RL, another promising direction is to investigate whether similar structural signals arise in other deep learning paradigms, such as supervised or self-supervised learning, where optimization dynamics and representation evolution may exhibit analogous patterns.

The probe batch is fixed throughout training in order to ensure comparability across runs. 
While this design is appropriate for the present study, alternative probe constructions or adaptive probes may reveal additional structural phenomena and deserve further investigation.

On the theoretical side, our analysis isolates the first-order flip mechanism linking LR, activation sign changes, and OUI dynamics. 
It does not model higher-order effects introduced by PPO-specific ingredients such as clipping, distributional shift, or value-function bootstrapping. 
Extending the theory to these factors is therefore an important direction for future work.

\medskip 

A particularly promising avenue is to move from diagnosis to control. 
The structural regimes identified here suggest adaptive optimization strategies in which actor and critic LRs are adjusted separately to keep the critic within a non-saturated OUI band while maintaining high actor OUI. 
More broadly, online LR adaptation driven by structural feedback may reduce the need for expensive hyperparameter sweeps and help stabilize RL training across a wider range of settings.

\section*{Acknowledgements}
This research was funded by the projects PID2023-146569NB-C21 and PID2023-146569NB-C22 supported by MICIU/AEI/10.13039/501100011033 and ERDF/UE. Alberto Fernández-Hernández was supported by the predoctoral grant PREP2023-001826 supported by MICIU/AEI/10.13039/501100011033 and ESF+. Cristian Pérez-Corral received support from the \textit{Conselleria de Educación, Cultura, Universidades y Empleo} (reference CIACIF/2024/412) through the European Social Fund Plus 2021–2027 (FSE+) program of the \textit{Comunitat Valenciana}. Manuel F. Dolz was supported by the Plan Gen--T grant CIDEXG/2022/13 of the \emph{Generalitat Valenciana}.

\bibliographystyle{IEEEtran}
\bibliography{references}

\end{document}